\documentclass{article}

\usepackage{microtype}
\usepackage{graphicx}
\usepackage{subfigure}
\usepackage{booktabs} 

\usepackage{hyperref}

\usepackage[accepted]{icml2025}

\usepackage{amsmath}
\usepackage{amssymb}
\usepackage{mathtools}
\usepackage{amsthm}
\usepackage{xcolor}
\usepackage{colortbl}
\definecolor{lightgray}{gray}{0.95}
\definecolor{mediumgray}{gray}{0.90}
\usepackage{tcolorbox}
\usepackage{enumitem}

\newcommand{\mathbbm}[1]{\text{\usefont{U}{bbm}{m}{n}#1}}

\usepackage{booktabs}      
\usepackage{multirow}      
\usepackage{graphicx}

\usepackage{enumitem} 
\usepackage[capitalize,noabbrev]{cleveref}

\theoremstyle{plain}
\newtheorem{theorem}{Theorem}[section]

\newtheorem{corollary}[theorem]{Corollary}
\theoremstyle{definition}
\newtheorem{definition}[theorem]{Definition}
\newtheorem{assumption}[theorem]{Assumption}
\theoremstyle{remark}

\usepackage{thmtools}
\usepackage{thm-restate}

\usepackage[textsize=tiny]{todonotes}

\icmltitlerunning{Submission and Formatting Instructions for ICML 2025}

\begin{document}

\twocolumn[
\icmltitle{Cross-Modal Memory Compression for Efficient Multi-Agent Debate}



\icmlsetsymbol{equal}{*}

\begin{icmlauthorlist}
\icmlauthor{Jing Wu}{equal}
\icmlauthor{Yue Sun}{equal}
\icmlauthor{Tianpei Xie}{}
\icmlauthor{Suiyao Chen}{}
\icmlauthor{Jingyuan Bao}{}
\icmlauthor{Yaopengxiao Xu}{}
\icmlauthor{Gaoyuan Du}{}
\icmlauthor{Inseok Heo}{}
\icmlauthor{Alexander Gutfraind}{}
\icmlauthor{Xin Wang}{}
\end{icmlauthorlist}


\icmlcorrespondingauthor{Firstname1 Lastname1}{first1.last1@xxx.edu}
\icmlcorrespondingauthor{Firstname2 Lastname2}{first2.last2@www.uk}

\icmlkeywords{Machine Learning, ICML}

\vskip 0.3in
]



\printAffiliationsAndNotice{\icmlEqualContribution} 

\begin{abstract}
Multi-agent debate can improve reasoning quality and reduce hallucinations, but it incurs rapidly growing context as debate rounds and agent count increase. Retaining full textual histories leads to token usage that can exceed context limits and often requires repeated summarization, adding overhead and compounding information loss. We introduce DebateOCR, a cross-modal compression framework that replaces long textual debate traces with compact image representations, which are then consumed through a dedicated vision encoder to condition subsequent rounds. This design compresses histories that commonly span tens to hundreds of thousands of tokens, cutting input tokens by more than 92\% and yielding substantially lower compute cost and faster inference across multiple benchmarks. We further provide a theoretical perspective showing that diversity across agents supports recovery of omitted information: although any single compressed history may discard details, aggregating multiple agents’ compressed views allows the collective representation to approach the information bottleneck with exponentially high probability.
\end{abstract}

\begin{figure}[t!]
\begin{center}
    \includegraphics[width=1\linewidth]{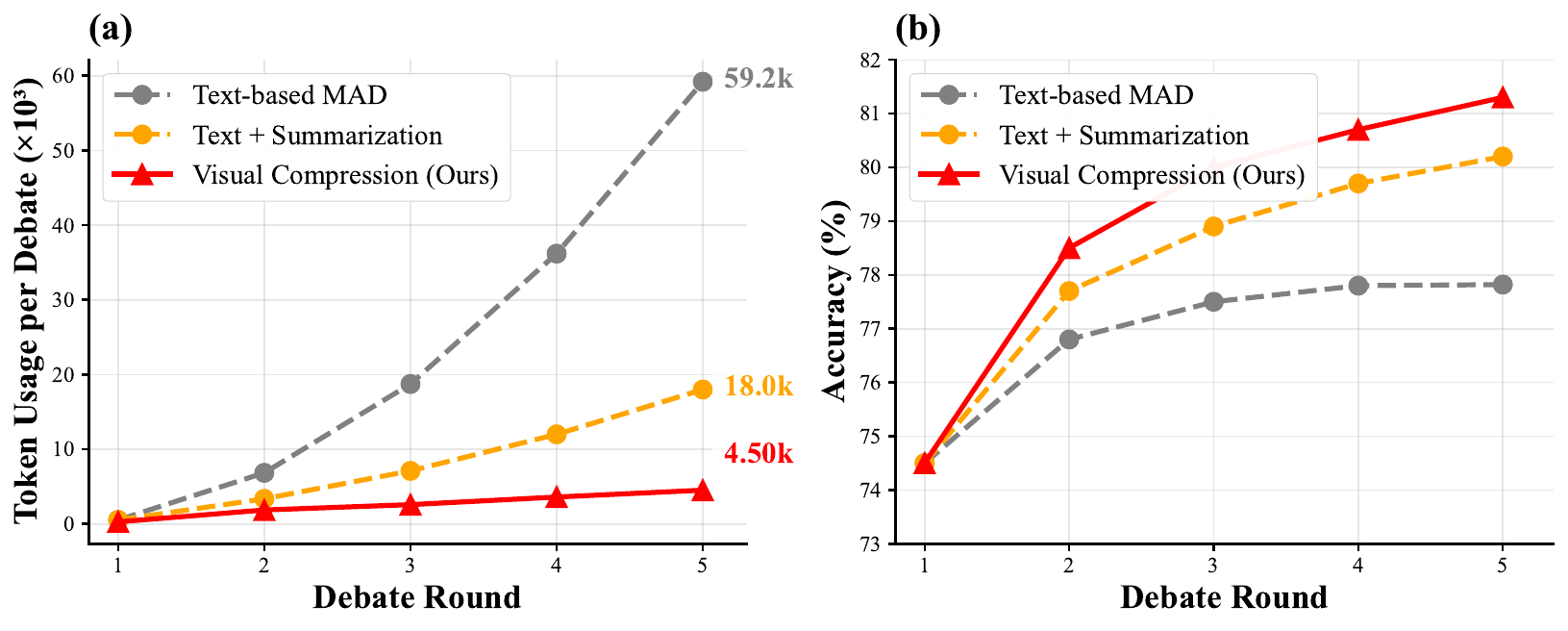}
\end{center}
\vspace{-5mm}
\caption{Visual compression addresses the token inefficiency of multi-agent debate. 
We compare three paradigms across 5 debate rounds using Qwen2-VL on GSM8K with 3 agents. 
(a) \textbf{Token consumption:} Text-based MAD accumulates 59.2K tokens by Round 5 
due to repeatedly storing debate history, while our visual compression reduces this 
by 92.2\% to only 4.5K tokens. Text with summarization achieves 69.6\% reduction to 18.0K tokens. 
(b) \textbf{Reasoning accuracy:} Despite lower computational cost, visual compression 
achieves the highest accuracy of 81.3\%, outperforming text with summarization
and text-based MAD.
\textit{Note:} At Round 5, text-based MAD (the gray curve) exceeds the context window limit. Therefore, the accuracy is measured with truncated debate history.}
\vspace{-5mm}
\label{fig:motivation}
\end{figure}

\section{Introduction}
\label{Intro}
Multi-agent debate (MAD) has emerged as a powerful paradigm for enhancing large language model (LLM) performance across reasoning, factuality, and complex problem-solving tasks\cite{du2023improving,khan2024debating,liu2025breaking,wu2025unfixing,liang2023encouraging}. By enabling multiple model instances to propose, critique, and refine responses through iterative discussion, MAD consistently outperforms single-agent approaches on mathematical reasoning \cite{cobbe2021training}, question answering \cite{hendrycks2020measuring}, and image captioning \cite{lin2014microsoft, wu2025building, lin2024interpreting}, etc. The fundamental principle underlying MAD's success is that diverse perspectives and iterative refinement converge toward more reliable solutions, similar to human collaborative problem-solving.

However, the computational overhead of MAD scales rapidly with both the number of agents and debate rounds as shown in Figure~\ref{fig:motivation}. Each agent must maintain complete debate histories as textual context, with token consumption growing quadratically as histories are replicated across all agents. As a result, extended debates frequently exceed context window limits, and lengthy debate histories complicate final decision-making as judges must extract relevant information from increasingly verbose exchanges. Recent analysis reveals that context limitations and communication breakdowns account for significant performance degradation in multi-agent systems, with agents struggling to maintain comprehensive state across extended interactions \cite{cemri2025multi}. Existing approaches address this through periodic summarization or truncation strategies \cite{chen2024reconcile,liu2025breaking,wu2025unfixing}, but these introduce additional computational cost and inference latency. To the best of our knowledge, no scalable solution exists for maintaining full debate context without prohibitive token overhead.

In this work, we propose a framework that applies visual compression to substantially reduce token consumption in multi-agent debate. Our key insight is that textual debate histories can be rendered as images and processed by specialized vision encoders through cross-modal operations, effectively converting text tokens into vision tokens at a fraction of the original cost. We design a compression-optimized vision encoder that maintains minimal activations under high-resolution inputs while achieving over 92\% token reduction. This approach fundamentally addresses the scalability challenge by converting the quadratically growing textual context into compact visual representations shared across agents.

Beyond efficiency, our framework offers several advantages: (1) it seamlessly integrates with existing MAD algorithms without architectural modifications; (2) it eliminates the need for summarization strategies, preserving complete debate histories; (3) visual encoding maintains richer contextual information than text alone, as vision tokens naturally capture the structural relationships and logical flow of debates, leading to improved reasoning quality.

We evaluate our framework on three reasoning benchmarks: MATH~\cite{hendrycks2020measuring}, 
GSM8K~\cite{cobbe2021training}, and GPQA~\cite{rein2024gpqa}, using four vision-language 
models: Qwen2.5-VL-7B \cite{bai2025qwen2}, Llama-3.2-11B-Vision \cite{meta2024llama}, InternVL2-8B \cite{chen2024internvl}, and Pixtral-12B \cite{agrawal2024pixtral}. The cross-modal compression method achieves over 92\% token reduction, while maintaining competitive accuracy. These gains 
require no modifications to the underlying debate algorithm or agent architecture.



To explain why compression preserves accuracy despite dramatic token reduction, we develop a theoretical analysis from the perspective of the Information Bottleneck~\cite{kawaguchi2023does}, which characterizes the minimum information required for optimal decision-making. We show that, as the number of diverse agents increases, compressed histories converge to this bottleneck. The central mechanism is information recovery through diversity: although each agent’s compressed history may discard some task-relevant information, different agents tend to preserve complementary aspects of the signal. When a majority of agents successfully retain the relevant information, aggregation recovers an essentially complete representation. At the same time, compression suppresses spurious artifacts, and independent artifacts introduced by different agents are further canceled through aggregation. Together, these effects explain why multi-agent systems can tolerate aggressive compression: collective redundancy offsets individual information loss, enabling accurate system-level decisions despite severe per-agent token reduction.

\section{Related Work}
\label{Related}

\subsection{Multi-Agent Debate for LLM Reasoning}
Multi-agent debate (MAD) enhances LLM reasoning by enabling multiple agents to propose, critique, and refine responses through iterative discussion \cite{du2023improving, liang2023encouraging, chan2023chateval}. Prior work explores role assignment strategies \cite{wang2023unleashing}, debate protocols that encourage error correction \cite{khan2024debating, liu2025breaking}, and expert-guided collaboration through meta-programming and consistency mechanisms \cite{hong2023metagpt, xiong2023examining, pham2023let}. Recent advances include reflective multi-agent collaboration \cite{bo2024reflective} and self-improvement through reinforcement learning \cite{chen2024self, subramaniam2025multiagent}. However, recent analysis reveals that context limitations and inter-agent misalignment cause significant failures, with agents struggling to maintain state across extended interactions \cite{cemri2025multi}. This stems from a fundamental bottleneck: debate histories grow quadratically with agents and rounds, frequently exceeding context windows and requiring computationally expensive summarization \cite{chen2024reconcile, wu2025unfixing}.

\subsection{Context Compression for LLMs}
Context compression addresses long-context challenges through multiple approaches. Vision-based methods convert text into images processed by lightweight encoders, achieving substantial compression ratios \cite{wei2025deepseek, xing2025vision}. Text-based methods employ soft prompt compression \cite{ge2023context, mu2023learning, chevalier2023adapting} or selection-based token pruning using information entropy \cite{li2023compressing, jiang2023llmlingua}. Memory-augmented architectures extend context through external memory banks \cite{mohtashami2023landmark, tworkowski2023focused}. However, vision-based compression for multi-agent systems remains unexplored. Our work differs by applying visual compression specifically to dynamic debate contexts rather than static documents, addressing the unique challenges of multi-agent communication efficiency.

\subsection{Vision-Language Models}
Vision-language models integrate visual and textual modalities for multimodal understanding \cite{zhang2024vision}. Key architectures include Qwen2-VL \cite{wang2024qwen2}, LLaVA-OneVision \cite{li2024llava}, InternVL2 \cite{chen2024far}, and DeepSeek-VL \cite{wu2024deepseek}, built upon vision encoders like CLIP \cite{radford2021learning}, ViT \cite{dosovitskiy2020image}, and SAM \cite{kirillov2023segment}. Recent work explores parameter-efficient adaptation and token compression for improved efficiency \cite{danish2025comprehensive}. We leverage pretrained vision encoders to compress textual debate contexts into compact visual representations, extending VLM capabilities to multi-agent communication.

\begin{figure*}[t!]
\begin{center}
    \includegraphics[width=0.9\linewidth]{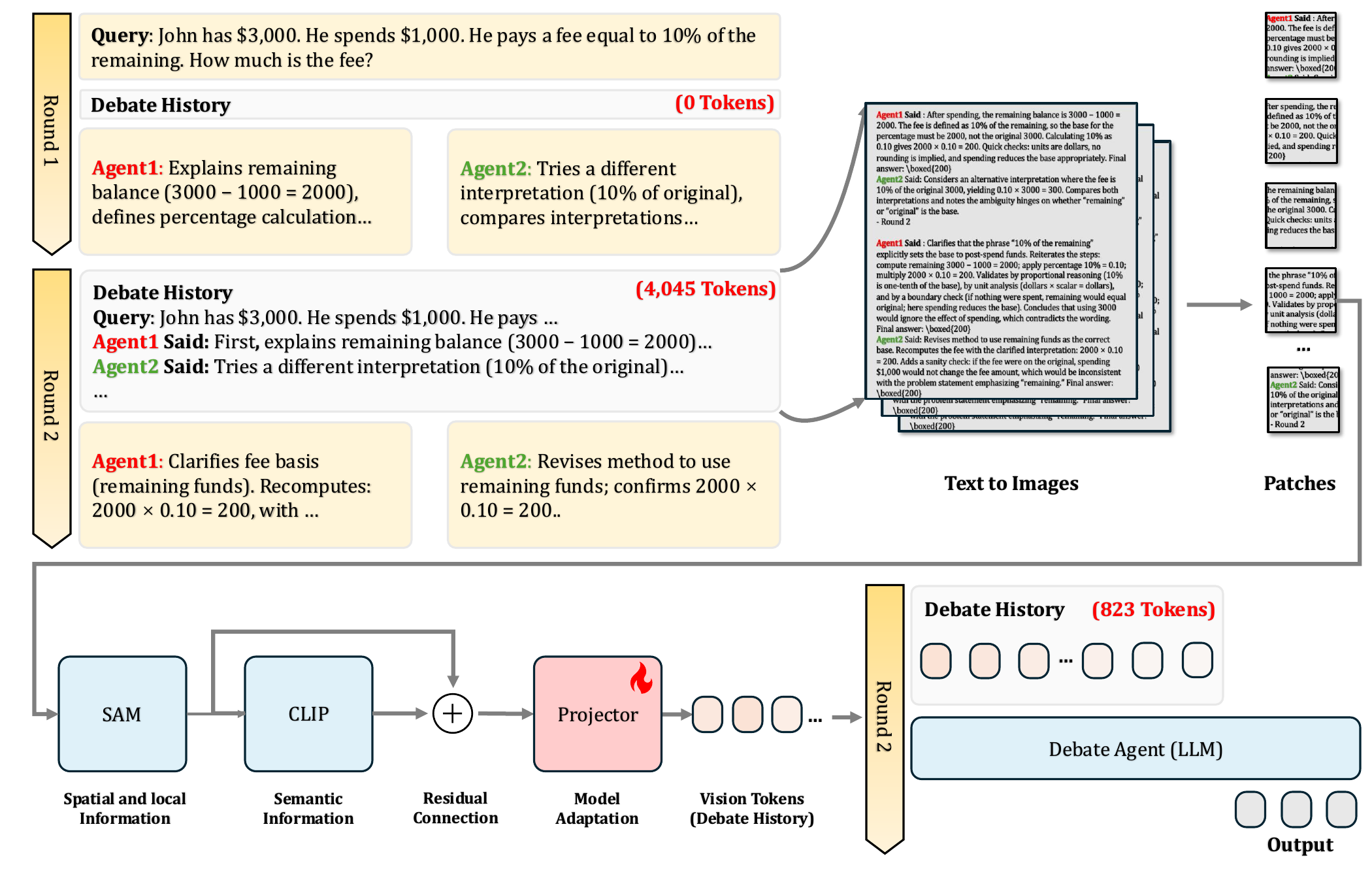}
\end{center}
\vspace{-5mm}
\caption{Visual compression framework for multi-agent debate. \textbf{Top:} Text-based debate accumulates token count across rounds. \textbf{Bottom:} Our approach converts debate history into visual representations. We first train a lightweight projector to adapt SAM-CLIP features to the vision embedding space of target MLLMs. During inference, debate history text is encoded as images via SAM, embedded through CLIP, and projected into small amount of vision tokens, achieving token reduction while preserving debate context for subsequent rounds.}
\label{fig:framework}
\vspace{-3mm}
\end{figure*}

\section{Method}
\label{sec:method}

We propose DebateOCR, a framework that leverages visual compression to address the scalability challenges of multi-agent debate, reducing token consumption over 90\%.

\subsection{Preliminaries and Computational Challenges}
\label{subsec:mad_preliminaries}

\textbf{Preliminaries.} 
We formalize multi-agent debate as a sequential state-expansion process involving 
a set of $K$ agents, denoted by $\mathcal{A} = \{A_1, \dots, A_K\}$. Each agent 
$A_i$ is instantiated as a LLM. Mathematically, we treat 
each agent as a policy $\pi_i$ that maps a query context to a textual response. 
Given a query $q \in \mathcal{Q}$ and a debate history $H_r = \{H_{i,r}\}_{i=1}^K$, the agent $i$ samples 
a solution $s_{i,r}$ at round $r$:
\begin{equation}
\label{eq:policy}
    s_{i,r} \sim \pi_i(\cdot \mid q, H_r),
\end{equation}
where $s_{i,r} \in \mathcal{S}$ represents the sequence of tokens generated by 
agent $i$ at round $r$, $\mathcal{S}$ denotes the output token space and $H_{i,r}$ denote the history generated at agent $i$ and round $r$ using the correspnding query.

The state of the system is defined by the accumulation of generated solutions. 
We define the history $H_r$ recursively. At the initial round ($r=1$), the 
history is empty: $H_1 = \emptyset$. For any subsequent round $r > 1$, the 
history is updated by appending the set of all solutions from the previous round:
\begin{equation}
\label{eq:history}
    H_r = H_{r-1} \cup \{ s_{1, r-1}, \dots, s_{K, r-1} \}.
\end{equation}
Here, $H_r$ serves as the global context observed by all $K$ agents to generate 
the next step of reasoning.

The debate concludes after a fixed horizon of $R$ rounds. The final output $y$ 
is derived by applying a consensus function $\phi$ over the final set of solutions:
\begin{equation}
    y = \phi\left( s_{1,R}, \dots, s_{K,R} \right),
\end{equation}
where $\phi$ typically represents a majority vote, model-based judge, or weighted 
aggregation. The choice of $\phi$ does not affect our analysis of computational 
costs.

\textbf{The Scalability Challenge.}
The primary bottleneck in this formulation is the quadratic growth of input 
context tokens. We quantify this cost by analyzing the total number of tokens 
processed across all agents and rounds.

Let $L$ denote the average length of a solution $s_{i,r}$ in tokens. The size 
of the history context at round $r$, denoted as $|H_r|$, scales linearly with 
the number of agents and rounds:
\begin{equation}
    |H_r| = K \cdot (r-1) \cdot L.
\end{equation}

However, the computational burden is multiplicative. Since all $K$ agents must 
process this history independently at every round, the token consumption at 
round $r$ is:
\begin{equation}
    C(r) = K \cdot |H_r| = K^2 \cdot (r-1) \cdot L.
\end{equation}

The cumulative token consumption across all $R$ rounds becomes:
\begin{equation}
\label{eq:quadratic_cost}
    C_{\text{total}} = \sum_{r=1}^{R} C(r) = K^2 L \sum_{r=1}^{R} (r-1) 
    = \frac{K^2 L R(R-1)}{2}.
\end{equation}

This implies a complexity of $\mathcal{O}(K^2 R^2 L)$ as both the number of 
agents and debate rounds scale. Figure~\ref{fig:motivation}(a) empirically 
demonstrates this quadratic scaling: a 5-round debate on GSM8K with 3 agents 
accumulates 59.5K tokens, exceeding typical context windows. This quadratic 
scaling severely limits the applicability of MAD to complex reasoning tasks 
requiring extended deliberation, as token budgets are quickly exhausted even 
with modest numbers of agents and rounds.

\subsection{Framework Overview}
\label{subsec:framework_overview}

To address the quadratic token scaling of Eq.~\ref{eq:quadratic_cost}, we propose 
a visual compression framework that converts textual debate histories into compact 
visual representations. The approach operates in two distinct phases: \textit{an offline 
training phase} that learns to align visual encodings with the target MLLM's 
embedding space, and an \textit{online inference phase} that compresses debate histories 
into a constant number of vision tokens.

\textbf{Two-Phase Pipeline.}
Our framework leverages a vision encoding strategy that combines 
complementary feature representations from pre-trained SAM and CLIP encoders. 
A lightweight adapter network learns to project their joint features into the 
target MLLM's vision token space. This adapter is trained once per target MLLM 
and transfers seamlessly across different debate scenarios without requiring 
scenario-specific fine-tuning.

\noindent\textbf{Training phase}: 
We train the adapter to reconstruct text from rendered images using a diverse 
corpus spanning multiple reasoning domains. The training objective encourages the 
adapter to learn a compressed representation that preserves the semantic content 
necessary for accurate text reconstruction. Details are included in Section~\ref{subsec:training_phase}.

\noindent\textbf{Inference phase}: 
During a multi-agent debate, we render the textual history as a structured image, 
encode it through the trained vision pipeline into a fixed-size sequence of vision 
tokens, and inject these tokens into each agent's context in place of the original 
text. This compression reduces token consumption from $\mathcal{O}(K^2 R^2 L)$ to 
$\mathcal{O}(K R N)$, where $N$ is a constant significantly smaller than $K R L$, 
achieving 80-97\% token reduction as demonstrated in Figure~\ref{fig:motivation}. We provide detailed description at Section~\ref{subsec:inference}.

\subsection{Training Phase}
\label{subsec:training_phase}

We train a lightweight adapter to align visual features from rendered debate histories 
with the target MLLM's embedding space. The training pipeline consists of frozen 
pre-trained encoders that extract visual features, followed by a trainable projection 
network that maps these features into the MLLM's token space. Crucially, only the 
adapter parameters are updated during training, where SAM, CLIP, and the MLLM remain 
completely frozen, enabling efficient training while leveraging strong pretrained 
representations.

\subsubsection{ENCODER PIPELINE}

The encoder pipeline processes rendered debate history images through a serial architecture 
combining SAM and CLIP encoders. Both SAM and CLIP remain frozen 
throughout training, following the design principles of DeepSeek-OCR \cite{wei2025deepseek}.

\textbf{SAM Feature Extraction.} Given a rendered image $I_r$ at $1024 \times 1024$ 
resolution, SAM-base first processes it using window-based attention with patch size 16. 
The SAM backbone outputs spatial features $F_{\text{SAM}} \in \mathbb{R}^{B \times 768 \times 64 \times 64}$ 
that capture fine-grained text layout and positioning information.

\textbf{Neck Module Architecture.} To align SAM features with CLIP's input requirements, 
we employ a Feature Pyramid Network-style neck module:
\begin{align*}
&F_{\text{SAM}} \xrightarrow{\text{Stage 1}} F_1 \xrightarrow{\text{Stage 2}} F_2 \xrightarrow{\text{Stage 3}} F_{\text{neck}} \\
&\mathbb{R}^{768 \times 64^2} \rightarrow \mathbb{R}^{256 \times 64^2} \rightarrow \mathbb{R}^{512 \times 32^2} \rightarrow \mathbb{R}^{1024 \times 16^2}
\end{align*}
Stage 1 applies $\text{Conv}_{1\times1}$, LayerNorm2d, and $\text{Conv}_{3\times3}$; 
Stages 2-3 apply $\text{Conv}_{3\times3}$ with stride 2, reducing resolution by $4\times$ 
while expanding to 1024 channels.

\textbf{CLIP Feature Extraction.} The neck module output serves as input to CLIP-Large, 
bypassing CLIP's standard patch embedding layer. To transform spatial features into the 
sequence format expected by CLIP transformers, we flatten the spatial dimensions and transpose 
to form a token sequence: $[B, 1024, 16, 16] \xrightarrow{\text{flatten}} [B, 1024, 256] \xrightarrow{\text{transpose}} [B, 256, 1024]$, 
yielding 256 tokens per image. CLIP-Large's learned positional embeddings are interpolated 
via bicubic interpolation from their original grid size to match the $16 \times 16$ spatial 
layout, then added to the feature sequence.  CLIP applies dense global attention to 
extract semantic representations, outputting $f_{\text{CLIP}} \in \mathbb{R}^{d_v}$ at each 
spatial position, where $d_v = 1024$.


\textbf{Feature Fusion with Residual Connection.} 
To preserve both spatial and semantic information, we fuse the neck-module features with the CLIP representation using a residual connection:
\[
\mathbf{f} = \mathbf{f}_{\text{CLIP}} + \mathbf{f}_{\text{neck}}, \quad \mathbf{f} \in \mathbb{R}^{d_v}
\]
where $d_v = 1024$. This residual fusion preserves fine grained spatial details from SAM while maintaining the semantic structure of CLIP embeddings, producing a compact representation for the subsequent adapter network.

\subsubsection{Adapter Network}
\label{subsubsec:adapter_network}

The adapter $\mathcal{P}_\theta$ is a lightweight projection network that transforms 
the encoder features into the target MLLM's embedding space. Given fused 
features $\mathbf{f} \in \mathbb{R}^{d_f}$ at each spatial position (where $d_f = d_v$), 
the adapter applies a two-layer MLP:
\begin{equation}
\label{eq:adapter_projection}
    \mathbf{z} = \mathcal{P}_\theta(\mathbf{f}) = \text{LayerNorm}(\mathbf{W}_2 \sigma(\mathbf{W}_1 \mathbf{f} + \mathbf{b}_1) + \mathbf{b}_2),
\end{equation}
where $\mathbf{W}_1 \in \mathbb{R}^{d_h \times d_f}$ projects to hidden dimension 
$d_h$, $\mathbf{W}_2 \in \mathbb{R}^{d \times d_h}$ projects to the MLLM's embedding 
dimension $d$, and $\sigma$ denotes GELU activation. Layer normalization is applied 
to the output for training stability. This produces a sequence of vision tokens $\mathbf{Z} = \{\mathbf{z}_1, \ldots, 
\mathbf{z}_N\} \in \mathbb{R}^{N \times d}$, where $N$ depends on image resolution.


The adapter's role as a lightweight feature projection layer, combined with training 
on diverse reasoning domains, enables it to transfer across different debate scenarios 
without requiring task-specific fine-tuning. Once trained for a target MLLM, the 
same adapter works across all debate protocols and topics.

\subsubsection{Training Procedure}
\label{subsubsec:training_procedure}

\textbf{Training Objective.}
We train the adapter through autoregressive text reconstruction. Given a text sample 
$T$, we render it as image $\mathcal{I}$, extract and project features to obtain 
vision tokens $\mathbf{Z}$, and minimize the cross-entropy loss between the MLLM's 
output and ground-truth text:
\begin{equation}
\label{eq:training_loss}
    \mathcal{L}(\theta) = -\sum_{t=1}^{|T|} \log p_{\text{MLLM}}(T_t \mid \mathbf{Z}),
\end{equation}
where $p_{\text{MLLM}}$ denotes the MLLM's generation probability conditioned on 
vision tokens. Only the adapter parameters $\theta = \{\mathbf{W}_1, \mathbf{W}_2, 
\mathbf{b}_1, \mathbf{b}_2\}$ are updated, while SAM, CLIP, and the MLLM remain frozen,
largely reducing computational cost while leveraging pretrained representations.

\subsection{Inference Phase}
\label{subsec:inference}

During multi-agent debate, we apply the trained framework to compress debate histories 
at each round $r$. The inference pipeline consists of three stages: rendering the 
textual history as a structured image, encoding the image through the frozen 
SAM-CLIP-adapter pipeline, and injecting the resulting vision tokens into each 
agent's context window.

\subsubsection{Text-to-Image Rendering}
\label{subsubsec:text_rendering}

At round $r$, the debate history $H_r = \{s_{i,t}\}_{i=1,t=1}^{K,r-1}$ contains 
all previous agent responses. We render $H_r$ as a structured image $\mathcal{I}_r$ 
that preserves agent identities and temporal ordering through spatial layout. Each 
agent's response is rendered with identifying markers (e.g., "Agent 1:", "Agent 2:") 
and rounds are visually separated. This spatial organization enables vision encoders 
to capture both response content and debate structure through visual parsing, which 
linear text concatenation cannot preserve.

The image resolution and layout balance text legibility for the encoders against 
the resulting number of vision tokens. Implementation details are provided in 
Section~\ref{sec:experiments}.

\subsubsection{Vision Encoding}
\label{subsubsec:vision_encoding}

The rendered image $\mathcal{I}_r$ is processed through the trained encoding pipeline: 
SAM extracts spatial features, CLIP builds semantic representations via the residual 
connection, and the adapter projects the combined features into vision tokens 
$\mathbf{Z}_r = \{\mathbf{z}_1, \ldots, \mathbf{z}_N\} \in \mathbb{R}^{N \times d}$.

The number of vision tokens $N$ remains constant for each round of debate history, 
enabling fixed-size compression of arbitrarily long debates. This reduces token 
consumption from $\mathcal{O}(K^2 R^2 L)$ to $\mathcal{O}(K R N)$.

\subsubsection{Context Injection}
\label{subsubsec:context_injection}

The vision tokens $\mathbf{Z}_r$ replace the textual debate history in each agent's 
context window. At round $r$, agent $i$ generates its response conditioned on the 
query and compressed history:
\begin{equation}
\label{eq:visual_generation}
    s_{i,r} \sim \pi_i(\cdot \mid q, \mathbf{Z}_r).
\end{equation}

The MLLM processes vision tokens identically to text tokens through its transformer 
layers, requiring no modification to the MLLM architecture or debate protocol. 


\section{Theoretical Analysis}
\label{sec: theoretical_analysis}

To understand why compression maintains accuracy despite dramatic token reduction, we provide a theoretical analysis from Information Bottleneck~\cite{kawaguchi2023does}. Our key result shows that multi-agent aggregation enables compressed histories to approach the information bottleneck, the optimal trade-off between preserving answer-relevant information and removing spurious artifacts.

Consider a multi-agent debate system with $K$ agents deliberating on query $q$ with ground truth $y_q$. Each agent $i$ generates a debate history $H_{i,r}$ through round $r$. Let $\mathcal{H}$ denote the collection of histories and $f : \mathcal{H}^K \to \mathcal{Y}$ the aggregation function (e.g., majority voting) that produces the final answer.

Debate histories contain two types of information: (1) \textbf{answer-relevant information} $I(H; Y_q)$ necessary for correct decisions, and (2) \textbf{artifacts} $V_i$ representing agent-specific styles, redundancies, and tangential explorations. 
Let $\mathcal{C} : \mathcal{H} \to \mathcal{Z}$ denote a compression function of the debating history with $|\mathcal{Z}| \ll |\mathcal{H}|$.

\begin{definition}[Information Bottleneck]
\label{def:bottleneck}
The \textbf{information bottleneck} is the minimum mutual information required for optimal decisions:
\begin{equation*}
I_{\text{bottleneck}} = \min_{\mathcal{H}} \{I(\mathcal{H}; Y_q) : \mathbb{E}[\ell(f(\mathcal{H}), y_q)] = \ell_{\min}\}
\end{equation*}
where $\ell_{\min} = \inf_{\mathcal{H}} \mathbb{E}[\ell(f(\mathcal{H}), y_q)]$ is the minimum achievable expected loss.
\end{definition}

Distance to the bottleneck is measured by:
\begin{equation}
D(\mathcal{H}) = |I(f(\mathcal{H}); Y_q) - I_{\text{bottleneck}}| + I(f(\mathcal{H}); V)
\end{equation}
where $I(f(\mathcal{H}); V) = \sum_{i=1}^K I(f(\mathcal{H}); V_i)$ captures total artifact content.

We establish two results: (1) compression improves decision quality by removing artifacts, and (2) multi-agent aggregation naturally converges to the information bottleneck.

\begin{restatable}{theorem}{ThmMain}
\label{thm:main}
Suppose: (1) Compression preserves information with probability $p > 0.5$: $\mathbb{P}(I(\mathcal{C}(H_i); Y_q) \geq I(H_i; Y_q) - \varepsilon) \geq p$; (2) Agents have diverse styles with conditionally independent artifacts $V_i \perp V_j \mid q$; (3) Compression reduces artifacts: $I(\mathcal{C}(H_i); V_i) \leq \gamma I(H_i; V_i)$ for $\gamma \in (0,1)$.
Then compressed histories approach the bottleneck with exponentially high probability:
\begin{equation}
\mathbb{P}(|I(f(\mathcal{C}(\mathcal{H})); Y_q) - I_{\text{bottleneck}}| \leq \varepsilon) \geq 1 - e^{-2K(p-\frac{1}{2})^2}
\end{equation}
while artifacts vanish: $I(f(\mathcal{C}(\mathcal{H})); V) = O(\gamma K)$. Consequently, $D(\mathcal{C}(\mathcal{H})) < D(\mathcal{H})$ when artifacts are substantial.
\end{restatable}

\begin{proof}[Proof sketch]
The key insight is that diverse agents cover complementary aspects. With probability $p > 0.5$, each agent preserves its information. By Hoeffding's inequality, the majority preserve information with probability $\geq 1 - e^{-2K(p-1/2)^2}$. Since agents cover different aspects (diversity), their union approaches complete coverage: $I(f(\mathcal{C}(\mathcal{H})); Y_q) \geq I_{\text{bottleneck}} - \varepsilon$. Meanwhile, artifacts are conditionally independent and reduced by factor $\gamma$, yielding $I(f(\mathcal{C}(\mathcal{H})); V) \leq \gamma \sum_i I(H_i; V_i)$. For complete debate histories, $I(f(\mathcal{H}); Y_q) \approx I_{\text{bottleneck}}$ but $I(f(\mathcal{H}); V) \geq \bar{I}_V := \frac{1}{K}\sum_i I(H_i; V_i)$. Thus $D(\mathcal{C}(\mathcal{H})) \leq \varepsilon + \gamma K \bar{I}_V < K\bar{I}_V \lesssim D(\mathcal{H})$ when $(1-\gamma)K\bar{I}_V > \varepsilon$. Full proof in Appendix~\ref{sect:proof}.
\end{proof}

\begin{table*}[t!]
\centering
\caption{Performance comparison across models and datasets after 5 debate rounds with 3 agents. 
We report accuracy (\%), average number of input tokens per sample (in thousands), and average 
inference time per sample (seconds). 
$\uparrow$: higher is better, $\downarrow$: lower is better. Best results per model are 
\textbf{bolded}.}
\label{tab:main_results}
\setlength{\tabcolsep}{5pt}
\small
\resizebox{\textwidth}{!}{%
\begin{tabular}{@{}ll@{\hskip 10pt}ccc@{\hskip 10pt}ccc@{\hskip 10pt}ccc@{}}
\toprule
\multirow{2}{*}{\textbf{Model}} & \multirow{2}{*}{\textbf{Method}} & 
\multicolumn{3}{c@{\hskip 10pt}}{\textbf{GSM8K}} & 
\multicolumn{3}{c@{\hskip 10pt}}{\textbf{MATH}} & 
\multicolumn{3}{c}{\textbf{GPQA}} \\
\cmidrule(lr){3-5} \cmidrule(lr){6-8} \cmidrule(lr){9-11}
& & Accuracy & \# Tokens & Time & Accuracy & \# Tokens & Time & Accuracy & \# Tokens & Time \\
& & $\uparrow$ (\%) & $\downarrow$ (K) & $\downarrow$ (s) & 
    $\uparrow$ (\%) & $\downarrow$ (K) & $\downarrow$ (s) & 
    $\uparrow$ (\%) & $\downarrow$ (K) & $\downarrow$ (s) \\
\midrule
\rowcolor{white}
InternVL-8B     & \textcolor{cyan}{T-MAD}       & 67.5  & 59.2  & 145.7 & 42.5 & 66.3 & 160.3 & 22.1 & 69.9 & 167.6 \\
\rowcolor{white}
                & \textcolor{red}{TS-MAD}       & 68.5  & 18.0  & 154.2 & 43.1 & 19.8 & 169.6 & \textbf{23.3} & 20.7 & 177.3 \\
\rowcolor{lightgray}
                & \textbf{DebateOCR (Ours)}   & \textbf{70.5} & \textbf{4.5} & \textbf{64.7} & 
                                         \textbf{44.2} & \textbf{4.7} & \textbf{71.2} & 
                                         {22.8} & \textbf{4.9} & \textbf{74.4} \\
\cmidrule(lr){2-11}
\rowcolor{white}
Qwen2.5-VL-7B   & \textcolor{cyan}{T-MAD}       & 77.6  & 59.2  & 131.1 & 44.8 & 72.0 & 144.2 & 32.6 & 75.9 & 150.8 \\
\rowcolor{white}
                & \textcolor{red}{TS-MAD}       & 80.2  & 19.6  & 138.8 & 45.1 & 21.5 & 152.7 & 33.8 & 22.5 & 159.6 \\
\rowcolor{lightgray}
                & \textbf{DebateOCR (Ours)}   & \textbf{81.3} & \textbf{4.5} & \textbf{58.2} & 
                                         \textbf{46.8} & \textbf{4.7} & \textbf{64.0} & 
                                         \textbf{34.8} & \textbf{4.9} & \textbf{66.9} \\
\cmidrule(lr){2-11}
\rowcolor{white}
Llama-3.2-11B   & \textcolor{cyan}{T-MAD}       & 78.2  & 66.8  & 218.6 & 45.2 & 74.8 & 240.5 & 33.0 & 78.9 & 251.4 \\
\rowcolor{white}
                & \textcolor{red}{TS-MAD}       & 78.8  & 20.3  & 231.3 & 46.7 & 22.3 & 254.4 & 33.9 & 23.4 & 265.8 \\
\rowcolor{lightgray}
                & \textbf{DebateOCR (Ours)}   & \textbf{79.8} & \textbf{4.5} & \textbf{97.1} & 
                                         \textbf{47.1} & \textbf{4.7} & \textbf{106.8} & 
                                         \textbf{34.4} & \textbf{4.9} & \textbf{111.7} \\
\cmidrule(lr){2-11}
\rowcolor{white}
Pixtral-12B     & \textcolor{cyan}{T-MAD}      & 79.4  & 62.6  & 233.1 & 49.2 & 70.1 & 256.4 & 36.4 & 73.9 & 268.1 \\
\rowcolor{white}
                & \textcolor{red}{TS-MAD}          & 81.5  & 19.1  & 246.7 & 51.8 & 20.9 & 271.4 & 37.9 & 21.9 & 283.7 \\
\rowcolor{lightgray}
                & \textbf{DebateOCR (Ours)}   & \textbf{82.8} & \textbf{4.5} & \textbf{103.5} & 
                                         \textbf{53.6} & \textbf{4.7} & \textbf{113.9} & 
                                         \textbf{38.3} & \textbf{4.9} & \textbf{119.0} \\
\bottomrule
\vspace{-7mm}
\end{tabular}
}
\end{table*}

\begin{corollary}[Sample Complexity]
\label{cor:sample_complexity}
To achieve $|I(f(\mathcal{C}(\mathcal{H})); Y_q) - I_{\text{bottleneck}}| \leq \varepsilon$ with confidence $1-\delta$, it suffices to use $K \geq \frac{\ln(1/\delta)}{2(p-1/2)^2}$ agents.
\end{corollary}

Theorem~\ref{thm:main} reveals a fundamental synergy between compression and aggregation. Although each individual agent incurs information loss under compression ($\varepsilon > 0$), collective aggregation recovers the missing information through agent diversity. Different agents preserve complementary aspects of signals: what one discards, another retains. This phenomenon formalizes the wisdom of crowds in information-theoretic terms: while individual representations are imperfect, their aggregation yields a near-optimal collective representation. Crucially, the exponential concentration guarantee shows that this effect improves steadily with the number of agents $K$, rather than relying solely on an asymptotic regime. As $K$ increases, the aggregated representation concentrates increasingly tightly around the information bottleneck. This explains why aggressive compression (e.g., substantial token reduction) can succeed in multi-agent systems while potentially failing in single-agent settings.


\section{Experiments}
\label{sec:experiments}
We evaluate DebateOCR on three mathematical reasoning benchmarks. Additional details for training, Adapter design, Rendering and Debate Configuration will be available at Appendix~\ref{sect:implementation}

\subsection{Experimental Setup}
\label{subsec:exp_setup}

\subsubsection{Tasks and Datasets}
\label{subsubsec:tasks_datasets}
We evaluate on three standard mathematical reasoning benchmarks:

\textbf{GSM8K} \cite{cobbe2021training} contains 8,500 grade school math word problems 
requiring arithmetic reasoning. We use the standard test set of 1,319 problems. Problems 
typically require 2-8 reasoning steps and involve topics such as percentages, ratios, 
and basic algebra.

\textbf{MATH} \cite{hendrycks2021measuring} comprises 12,500 competition-level mathematics 
problems spanning algebra, geometry, number theory, and calculus. We evaluate on the 
5,000-problem test set. 


\textbf{GPQA} \cite{rein2024gpqa} is a graduate-level science question-answering benchmark consisting of 448 multiple-choice questions written by domain experts in biology, physics, and chemistry.

\subsubsection{Models and Baselines}
\label{subsubsec:models_baselines}
We evaluate four open-source multimodal large language models with diverse architectures:

\noindent\textbf{Qwen2.5-VL-7B-Instruct} \cite{bai2025qwen2}: A 7B-parameter vision-language 
model with adaptive resolution encoding, supporting variable aspect ratios and dynamic 
token allocation.

\noindent\textbf{Llama-3.2-11B-Vision} \cite{meta2024llama}: An 11B-parameter model combining 
the Llama-3.2 language model with a vision encoder, optimized for multimodal instruction 
following.

\noindent\textbf{InternVL-8B} \cite{chen2024internvl}: A 8B-parameter model using tile-based 
high-resolution image processing, designed for fine-grained visual understanding.

\noindent\textbf{Pixtral-12B } \cite{agrawal2024pixtral}: A 12B-parameter multimodal model from  Mistral AI featuring a 400M-parameter vision encoder with variable image resolution 
support, capable of processing arbitrary numbers of images at their natural resolution 
and aspect ratio.

We compare three approaches to managing debate history:
\noindent\textbf{The Text-based Debate (\textcolor{cyan}{T-MAD})}: Standard text-based multi-agent debate where the complete 
textual history $H_r$ is provided to each agent at every round. This represents the 
baseline MAD approach without any compression \cite{khan2024debating, du2023improving}.

\noindent\textbf{The Text-based Debate with Summarization (\textcolor{red}{TS-MAD})}: After each round, the debate history is 
compressed via extractive and abstractive summarization to reduce token count while 
preserving key arguments. Agents receive the summarized history in subsequent rounds 
rather than the full text\cite{chen2024reconcile, liu2025breaking}.

\subsubsection{Evaluation Metrics}
\label{subsubsec:metrics}

We evaluate methods along three dimensions:

\noindent\textbf{Accuracy}: We report exact match accuracy on each benchmark. For GSM8K, GPQA
and MATH.

\noindent\textbf{Token Consumption}: We measure the total number of tokens processed across all agents and rounds. For text-based methods, this includes all input tokens, including query and debate history, at each round. For the proposed method, this includes the query tokens plus vision tokens from the compressed history.

\noindent\textbf{Inference Time}: We measure wall-clock time from initial query to final 
consensus, including all model forward passes across agents and rounds, on a single 
NVIDIA A100 GPU.

\subsection{Main Results}
\label{subsec:main_results}

Table~\ref{tab:main_results} presents our main results across four vision-language models 
on three reasoning datasets. We compare our visual compression approach against two baselines: 
pure text-based debate and debate with extractive summarization. All methods use 5 debate 
rounds with 3 agents. Our visual compression achieves the best or competitive accuracy in 
most settings while dramatically reducing token consumption and inference time.

\textbf{Accuracy.}
DebateOCR achieves the best accuracy across most experimental settings. On InternVL 8B, it achieves 70.5\% on GSM8K and 44.2\% on MATH, outperforming both baselines. On Qwen2.5 VL 7B, it reaches 81.3\% on GSM8K, 46.8\% on MATH, and 34.8\% on GPQA, consistently surpassing text-based debate and summarization. Similar improvements are observed for Llama 3.2 11B and Pixtral 12B, where DebateOCR achieves 79.8\% and 82.8\% on GSM8K, respectively. The accuracy gains arise from improved preservation of spatial structure and formatting in rendered images, which helps models track multi-agent exchanges more effectively than long token sequences. In one case—GPQA with InternVL 8B—summarization slightly outperforms our method with 23.3\% compared to 22.8\%, likely because extractive summaries better retain domain-specific terminology for graduate-level questions.

\begin{table}[t]
\centering
\caption{Ablation study on image resolution on MATH with 5 rounds and 3 agents using Qwen2.5-VL-7B.}
\vspace{1mm}
\label{tab:resolution_ablation}
\resizebox{0.8\columnwidth}{!}{%
\begin{tabular}{cccc}
\toprule
\textbf{Resolution} & \textbf{\# Tokens} & \textbf{Acc (\%)} & \textbf{Compress} \\
\midrule
$224 \times 224$   &  16   & 71.2 & $\mathbf{19.2\times}$ \\
$336 \times 336$   &  36   & 73.1 & 18.9$\times$ \\
$448 \times 448$   &  49   & 75.5 & 18.6$\times$ \\
$512 \times 512$   &  64   & 76.2 & 18.3$\times$ \\
$1024 \times 1024$ &  256  & 76.3 & 15.3$\times$ \\
$1536 \times 1536$ &  576  & 76.3 & 12.0$\times$ \\
$2048 \times 2048$ &  1024 & 76.6 & 9.2$\times$ \\
\bottomrule
\vspace{-15mm}
\end{tabular}
}
\end{table}

\textbf{Token Efficiency.}
DebateOCR achieves substantial reductions compared to both baselines. On InternVL-8B with GSM8K, our approach uses only 4.5K tokens, versus 59.2K for text-based debate and 18.0K for summarization, corresponding to 92.4\% and 75.0\% reductions, respectively. Token usage remains nearly constant across datasets, increasing only slightly with question length. In contrast, text-based methods exhibit significant token growth with debate rounds, consuming 59.2K–78.9K tokens for text-based debate and 18.0K–23.4K tokens for summarization, depending on dataset complexity and model.

\textbf{Inference Speed.}
DebateOCR delivers substantial inference speedups over both baselines. On InternVL-8B with GSM8K, inference completes in 64.7 s, compared to 145.7 s for text-based debate and 154.2 s for summarization, yielding 2.25× and 2.38× speedups, respectively. Notably, the summarization baseline is consistently slower than pure text despite using fewer tokens, as it must generate summaries after each round, whereas our rendering incurs a fixed computational cost. These speedups are consistent across models and datasets, with larger models such as Llama-3.2-11B and Pixtral-12B exhibiting greater absolute time savings due to higher per-token costs. For example, on Llama-3.2-11B with MATH, DebateOCR reduces inference time from 240.5 s to 106.8 s, achieving a 2.25× speedup.

\textbf{Discussion for Accuracy Gains.}
Multi-round text-based debates could accumulate substantial noise, including redundant arguments, overthinking patterns where agents explore but abandon incorrect reasoning paths, and agent-specific stylistic variations that inflate context without improving decision quality. These artifacts increase as $I(H_r; D_i)$ grows quadratically with debate rounds, introducing variability in the aggregation process. Visual compression may substantially reduce this artifact information to $I(C(H_r); D_i) \ll I(H_r; D_i)$ by rendering debate histories in standardized spatial layouts that naturally suppress stylistic variations and extract core logical structure. This explains why DebateOCR achieves 
higher accuracy despite dramatically reduced token counts. Additional proofs and discussion are available at Appendix ~\ref{sect:proof}.

\subsection{Ablation Studies}
\label{subsec:ablation_studies}

\textbf{Image Resolution and Token Budget.}
We conduct ablation studies on Qwen2.5-VL-7B with the MATH dataset using 3 agents and 5 
rounds of debate in Table~\ref{tab:resolution_ablation}. Higher resolutions produce more vision tokens, enabling better text legibility but increasing computational cost. We observe that 1024×1024 resolution achieves the best balance, producing approximately 256 vision tokens with competitive accuracy at 76.3\%. Lower resolutions such as 224×224 reduce vision tokens to 16 but suffer accuracy degradation due to insufficient text clarity.


\begin{table}[t]
\centering
\caption{Comparison of vision encoders on MATH with Qwen2.5-VL-7B using 5 rounds and 3 agents.}
\vspace{1mm}
\label{tab:encoder_comparison}
\resizebox{\columnwidth}{!}{%
\begin{tabular}{@{}lcc@{}}
\toprule
\textbf{Metric} & \textbf{DebateOCR} & \textbf{QwenVL2.5} \\
\midrule
\quad Resolution & $1024 \times 1024$ (fixed) & $1036 \times 1036$ (dynamic) \\
\quad Vision Tokens (per image) & 256 & 1,369 \\
\midrule
\quad Total Vision Tokens & 3.8K & 20.5K \\
\quad Token Reduction & \textbf{5.4$\times$} & --- \\
\quad Inference Time (per turn) & 3.0s & 4.2s \\
\quad Total Inference Time & 64.0s & 89.3s \\
\midrule
\quad Accuracy (\%) & 46.8  &46.9 \\
\bottomrule
\vspace{-13mm}
\end{tabular}%

}
\end{table}

\textbf{Comparison with Native Vision Encoder.}
Table~\ref{tab:encoder_comparison} compares our approach with Qwen2-VL's original vision encoder. DebateOCR achieves 5.4× token reduction with 3.8K versus 20.5K vision 
tokens over the full debate, while maintaining comparable accuracy. This efficiency translates 
to faster inference at 64.0s versus 89.3s, demonstrating that compact visual representations preserve essential reasoning information with lower computational cost.

\section{Conclusion}
\label{sec:conclusion}

We introduced DebateOCR, a token compression framework for multi-agent debate that addresses the 
quadratic token growth problem by rendering debate histories as images. The method achieves 
over 92\% token reduction compared to text-based debate while maintaining competitive 
accuracy across mathematical and scientific reasoning tasks. The approach generalizes 
effectively across diverse MLLMs and scales linearly with debate rounds 
and agent count. We theoretically explain why compression preserves essential 
reasoning information while filtering debate artifacts. This work demonstrates that visual 
representations offer a practical and efficient alternative to text-based communication in 
multi-agent systems, enabling scalable debates without modifications to underlying algorithms 
or architectures.

\bibliography{example_paper}

@inproceedings{lin2014microsoft,
  title={Microsoft coco: Common objects in context},
  author={Lin, Tsung-Yi and Maire, Michael and Belongie, Serge and Hays, James and Perona, Pietro and Ramanan, Deva and Doll{\'a}r, Piotr and Zitnick, C Lawrence},
  booktitle={European conference on computer vision},
  pages={740--755},
  year={2014},
  organization={Springer}
}

@article{loshchilov2017decoupled,
  title={Decoupled weight decay regularization},
  author={Loshchilov, Ilya and Hutter, Frank},
  journal={arXiv preprint arXiv:1711.05101},
  year={2017}
}

@article{subramaniam2025multiagent,
  title={Multiagent finetuning: Self improvement with diverse reasoning chains},
  author={Subramaniam, Vighnesh and Du, Yilun and Tenenbaum, Joshua B and Torralba, Antonio and Li, Shuang and Mordatch, Igor},
  journal={arXiv preprint arXiv:2501.05707},
  year={2025}
}

@inproceedings{du2023improving,
  title={Improving factuality and reasoning in language models through multiagent debate},
  author={Du, Yilun and Li, Shuang and Torralba, Antonio and Tenenbaum, Joshua B and Mordatch, Igor},
  booktitle={Forty-first International Conference on Machine Learning},
  year={2023}
}

@article{liang2023encouraging,
  title={Encouraging divergent thinking in large language models through multi-agent debate},
  author={Liang, Tian and He, Zhiwei and Jiao, Wenxiang and Wang, Xing and Wang, Yan and Wang, Rui and Yang, Yujiu and Shi, Shuming and Tu, Zhaopeng},
  journal={arXiv preprint arXiv:2305.19118},
  year={2023}
}

@article{chan2023chateval,
  title={Chateval: Towards better llm-based evaluators through multi-agent debate},
  author={Chan, Chi-Min and Chen, Weize and Su, Yusheng and Yu, Jianxuan and Xue, Wei and Zhang, Shanghang and Fu, Jie and Liu, Zhiyuan},
  journal={arXiv preprint arXiv:2308.07201},
  year={2023}
}

@article{khan2024debating,
  title={Debating with more persuasive llms leads to more truthful answers},
  author={Khan, Akbir and Hughes, John and Valentine, Dan and Ruis, Laura and Sachan, Kshitij and Radhakrishnan, Ansh and Grefenstette, Edward and Bowman, Samuel R and Rockt{\"a}schel, Tim and Perez, Ethan},
  journal={arXiv preprint arXiv:2402.06782},
  year={2024}
}

@article{wang2023unleashing,
  title={Unleashing the emergent cognitive synergy in large language models: A task-solving agent through multi-persona self-collaboration},
  author={Wang, Zhenhailong and Mao, Shaoguang and Wu, Wenshan and Ge, Tao and Wei, Furu and Ji, Heng},
  journal={arXiv preprint arXiv:2307.05300},
  year={2023}
}

@article{hong2023metagpt,
  title={Metagpt: Meta programming for multi-agent collaborative framework},
  author={Hong, Sirui and Zheng, Xiawu and Chen, Jonathan and Cheng, Yuheng and Wang, Jinlin and Zhang, Ceyao and Wang, Zili and Yau, Steven Ka Shing and Lin, Zijuan and Zhou, Liyang and others},
  journal={arXiv preprint arXiv:2308.00352},
  volume={3},
  number={4},
  pages={6},
  year={2023}
}

@article{xiong2023examining,
  title={Examining inter-consistency of large language models collaboration: An in-depth analysis via debate},
  author={Xiong, Kai and Ding, Xiao and Cao, Yixin and Liu, Ting and Qin, Bing},
  journal={arXiv preprint arXiv:2305.11595},
  year={2023}
}

@article{pham2023let,
  title={Let models speak ciphers: Multiagent debate through embeddings},
  author={Pham, Chau and Liu, Boyi and Yang, Yingxiang and Chen, Zhengyu and Liu, Tianyi and Yuan, Jianbo and Plummer, Bryan A and Wang, Zhaoran and Yang, Hongxia},
  journal={arXiv preprint arXiv:2310.06272},
  year={2023}
}

@article{chen2024self,
  title={Self-play fine-tuning converts weak language models to strong language models},
  author={Chen, Zixiang and Deng, Yihe and Yuan, Huizhuo and Ji, Kaixuan and Gu, Quanquan},
  journal={arXiv preprint arXiv:2401.01335},
  year={2024}
}

@misc{
c,
title={{LLM} Spark: Critical Thinking Evaluation of Large Language Models},
author={Runing Yang and Adam Nguyen and Hoang Anh Just and Ruoxi Jia and Ming Jin},
year={2025},
url={https://openreview.net/forum?id=0sJ8TqOLGS}
}

@article{hendrycks2021measuring,
  title={Measuring mathematical problem solving with the math dataset},
  author={Hendrycks, Dan and Burns, Collin and Kadavath, Saurav and Arora, Akul and Basart, Steven and Tang, Eric and Song, Dawn and Steinhardt, Jacob},
  journal={arXiv preprint arXiv:2103.03874},
  year={2021}
}

@article{cobbe2021training,
  title={Training verifiers to solve math word problems},
  author={Cobbe, Karl and Kosaraju, Vineet and Bavarian, Mohammad and Chen, Mark and Jun, Heewoo and Kaiser, Lukasz and Plappert, Matthias and Tworek, Jerry and Hilton, Jacob and Nakano, Reiichiro and others},
  journal={arXiv preprint arXiv:2110.14168},
  year={2021}
}

@inproceedings{rein2024gpqa,
  title={Gpqa: A graduate-level google-proof q\&a benchmark},
  author={Rein, David and Hou, Betty Li and Stickland, Asa Cooper and Petty, Jackson and Pang, Richard Yuanzhe and Dirani, Julien and Michael, Julian and Bowman, Samuel R},
  booktitle={First Conference on Language Modeling},
  year={2024}
}

@inproceedings{
liu2025breaking,
title={Breaking Mental Set to Improve Reasoning through Diverse Multi-Agent Debate},
author={Yexiang Liu and Jie Cao and Zekun Li and Ran He and Tieniu Tan},
booktitle={The Thirteenth International Conference on Learning Representations},
year={2025},
url={https://openreview.net/forum?id=t6QHYUOQL7}
}

@article{wu2025unfixing,
  title={Unfixing the mental set: Granting early-stage reasoning freedom in multi-agent debate},
  author={Wu, Jing and Chen, Suiyao and Heo, Inseok and Gutfraind, Sasha and Liu, Shengjie and Li, Chen and Srinivasan, Bharathi and Zhang, Xian and Sharps, Michael},
  year={2025}
}

@article{hendrycks2020measuring,
  title={Measuring massive multitask language understanding},
  author={Hendrycks, Dan and Burns, Collin and Basart, Steven and Zou, Andy and Mazeika, Mantas and Song, Dawn and Steinhardt, Jacob},
  journal={arXiv preprint arXiv:2009.03300},
  year={2020}
}

@article{lin2024interpreting,
  title={Interpreting and mitigating hallucination in mllms through multi-agent debate},
  author={Lin, Zheng and Niu, Zhenxing and Wang, Zhibin and Xu, Yinghui},
  journal={arXiv preprint arXiv:2407.20505},
  year={2024}
}

@inproceedings{wu2025building,
  title={Building more accountable multi-modal LLMs through spatially-informed visual reasoning},
  author={Wu, Jing and Chen, Suiyao and Gutfraind, Alexander and Heo, Inseok and Liu, Shengjie and Li, Chen and Curuksu, Jeremy and Sharps, Michael},
  booktitle={NeurIPS 2025 Workshop on Evaluating the Evolving LLM Lifecycle: Benchmarks, Emergent Abilities, and Scaling},
  year={2025}
}

@inproceedings{chen2024reconcile,
  title={Reconcile: Round-table conference improves reasoning via consensus among diverse llms},
  author={Chen, Justin and Saha, Swarnadeep and Bansal, Mohit},
  booktitle={Proceedings of the 62nd Annual Meeting of the Association for Computational Linguistics (Volume 1: Long Papers)},
  pages={7066--7085},
  year={2024}
}

@article{cemri2025multi,
  title={Why do multi-agent llm systems fail?},
  author={Cemri, Mert and Pan, Melissa Z and Yang, Shuyi and Agrawal, Lakshya A and Chopra, Bhavya and Tiwari, Rishabh and Keutzer, Kurt and Parameswaran, Aditya and Klein, Dan and Ramchandran, Kannan and others},
  journal={arXiv preprint arXiv:2503.13657},
  year={2025}
}

@article{bo2024reflective,
  title={Reflective multi-agent collaboration based on large language models},
  author={Bo, Xiaohe and Zhang, Zeyu and Dai, Quanyu and Feng, Xueyang and Wang, Lei and Li, Rui and Chen, Xu and Wen, Ji-Rong},
  journal={Advances in Neural Information Processing Systems},
  volume={37},
  pages={138595--138631},
  year={2024}
}

@article{wei2025deepseek,
  title={Deepseek-ocr: Contexts optical compression},
  author={Wei, Haoran and Sun, Yaofeng and Li, Yukun},
  journal={arXiv preprint arXiv:2510.18234},
  year={2025}}

@article{xing2025vision,
  title={Vision-centric Token Compression in Large Language Model},
  author={Xing, Ling and Wang, Alex Jinpeng and Yan, Rui and Shu, Xiangbo and Tang, Jinhui},
  journal={arXiv preprint arXiv:2502.00791},
  year={2025}
}

@article{jiang2023llmlingua,
  title={Llmlingua: Compressing prompts for accelerated inference of large language models},
  author={Jiang, Huiqiang and Wu, Qianhui and Lin, Chin-Yew and Yang, Yuqing and Qiu, Lili},
  journal={arXiv preprint arXiv:2310.05736},
  year={2023}
}

@inproceedings{li2023compressing,
  title={Compressing context to enhance inference efficiency of large language models},
  author={Li, Yucheng and Dong, Bo and Guerin, Frank and Lin, Chenghua},
  booktitle={Proceedings of the 2023 conference on empirical methods in natural language processing},
  pages={6342--6353},
  year={2023}
}

@article{ge2023context,
  title={In-context autoencoder for context compression in a large language model},
  author={Ge, Tao and Hu, Jing and Wang, Lei and Wang, Xun and Chen, Si-Qing and Wei, Furu},
  journal={arXiv preprint arXiv:2307.06945},
  year={2023}
}

@article{mu2023learning,
  title={Learning to compress prompts with gist tokens},
  author={Mu, Jesse and Li, Xiang and Goodman, Noah},
  journal={Advances in Neural Information Processing Systems},
  volume={36},
  pages={19327--19352},
  year={2023}
}

@article{chevalier2023adapting,
  title={Adapting language models to compress contexts},
  author={Chevalier, Alexis and Wettig, Alexander and Ajith, Anirudh and Chen, Danqi},
  journal={arXiv preprint arXiv:2305.14788},
  year={2023}
}

@article{mohtashami2023landmark,
  title={Landmark attention: Random-access infinite context length for transformers},
  author={Mohtashami, Amirkeivan and Jaggi, Martin},
  journal={arXiv preprint arXiv:2305.16300},
  year={2023}
}

@article{tworkowski2023focused,
  title={Focused transformer: Contrastive training for context scaling},
  author={Tworkowski, Szymon and Staniszewski, Konrad and Pacek, Miko{\l}aj and Wu, Yuhuai and Michalewski, Henryk and Mi{\l}o{\'s}, Piotr},
  journal={Advances in neural information processing systems},
  volume={36},
  pages={42661--42688},
  year={2023}
}

@article{zhang2024vision,
  title={Vision-language models for vision tasks: A survey},
  author={Zhang, Jingyi and Huang, Jiaxing and Jin, Sheng and Lu, Shijian},
  journal={IEEE transactions on pattern analysis and machine intelligence},
  volume={46},
  number={8},
  pages={5625--5644},
  year={2024},
  publisher={IEEE}
}

@article{wang2024qwen2,
  title={Qwen2-vl: Enhancing vision-language model's perception of the world at any resolution},
  author={Wang, Peng and Bai, Shuai and Tan, Sinan and Wang, Shijie and Fan, Zhihao and Bai, Jinze and Chen, Keqin and Liu, Xuejing and Wang, Jialin and Ge, Wenbin and others},
  journal={arXiv preprint arXiv:2409.12191},
  year={2024}
}

@article{li2024llava,
  title={Llava-onevision: Easy visual task transfer},
  author={Li, Bo and Zhang, Yuanhan and Guo, Dong and Zhang, Renrui and Li, Feng and Zhang, Hao and Zhang, Kaichen and Zhang, Peiyuan and Li, Yanwei and Liu, Ziwei and others},
  journal={arXiv preprint arXiv:2408.03326},
  year={2024}
}

@article{chen2024far,
  title={How far are we to gpt-4v? closing the gap to commercial multimodal models with open-source suites},
  author={Chen, Zhe and Wang, Weiyun and Tian, Hao and Ye, Shenglong and Gao, Zhangwei and Cui, Erfei and Tong, Wenwen and Hu, Kongzhi and Luo, Jiapeng and Ma, Zheng and others},
  journal={Science China Information Sciences},
  volume={67},
  number={12},
  pages={220101},
  year={2024},
  publisher={Springer}
}

@article{wu2024deepseek,
  title={Deepseek-vl2: Mixture-of-experts vision-language models for advanced multimodal understanding},
  author={Wu, Zhiyu and Chen, Xiaokang and Pan, Zizheng and Liu, Xingchao and Liu, Wen and Dai, Damai and Gao, Huazuo and Ma, Yiyang and Wu, Chengyue and Wang, Bingxuan and others},
  journal={arXiv preprint arXiv:2412.10302},
  year={2024}
}

@inproceedings{radford2021learning,
  title={Learning transferable visual models from natural language supervision},
  author={Radford, Alec and Kim, Jong Wook and Hallacy, Chris and Ramesh, Aditya and Goh, Gabriel and Agarwal, Sandhini and Sastry, Girish and Askell, Amanda and Mishkin, Pamela and Clark, Jack and others},
  booktitle={International conference on machine learning},
  pages={8748--8763},
  year={2021},
  organization={PmLR}
}

@article{dosovitskiy2020image,
  title={An image is worth 16x16 words: Transformers for image recognition at scale},
  author={Dosovitskiy, Alexey},
  journal={arXiv preprint arXiv:2010.11929},
  year={2020}
}

@inproceedings{kirillov2023segment,
  title={Segment anything},
  author={Kirillov, Alexander and Mintun, Eric and Ravi, Nikhila and Mao, Hanzi and Rolland, Chloe and Gustafson, Laura and Xiao, Tete and Whitehead, Spencer and Berg, Alexander C and Lo, Wan-Yen and others},
  booktitle={Proceedings of the IEEE/CVF international conference on computer vision},
  pages={4015--4026},
  year={2023}
}

@article{danish2025comprehensive,
  title={A comprehensive survey of Vision-Language Models: Pretrained models, fine-tuning, prompt engineering, adapters, and benchmark datasets},
  author={Danish, Sufyan and Sadeghi-Niaraki, Abolghasem and Khan, Samee Ullah and Dang, L Minh and Tightiz, Lilia and Moon, Hyeonjoon},
  journal={Information Fusion},
  pages={103623},
  year={2025},
  publisher={Elsevier}
}

@article{bai2025qwen2,
  title={Qwen2. 5-vl technical report},
  author={Bai, Shuai and Chen, Keqin and Liu, Xuejing and Wang, Jialin and Ge, Wenbin and Song, Sibo and Dang, Kai and Wang, Peng and Wang, Shijie and Tang, Jun and others},
  journal={arXiv preprint arXiv:2502.13923},
  year={2025}
}

@article{meta2024llama,
  title={Llama 3.2: Revolutionizing edge ai and vision with open, customizable models},
  author={Meta, AI},
  journal={Meta AI Blog. Retrieved December},
  volume={20},
  pages={2024},
  year={2024}
}

@inproceedings{chen2024internvl,
  title={Internvl: Scaling up vision foundation models and aligning for generic visual-linguistic tasks},
  author={Chen, Zhe and Wu, Jiannan and Wang, Wenhai and Su, Weijie and Chen, Guo and Xing, Sen and Zhong, Muyan and Zhang, Qinglong and Zhu, Xizhou and Lu, Lewei and others},
  booktitle={Proceedings of the IEEE/CVF conference on computer vision and pattern recognition},
  pages={24185--24198},
  year={2024}
}

@inproceedings{amini2019mathqa,
  title={Mathqa: Towards interpretable math word problem solving with operation-based formalisms},
  author={Amini, Aida and Gabriel, Saadia and Lin, Shanchuan and Koncel-Kedziorski, Rik and Choi, Yejin and Hajishirzi, Hannaneh},
  booktitle={Proceedings of the 2019 conference of the North American chapter of the association for computational linguistics: Human language technologies, volume 1 (long and short papers)},
  pages={2357--2367},
  year={2019}
}

@article{rajpurkar2016squad,
  title={Squad: 100,000+ questions for machine comprehension of text},
  author={Rajpurkar, Pranav and Zhang, Jian and Lopyrev, Konstantin and Liang, Percy},
  journal={arXiv preprint arXiv:1606.05250},
  year={2016}
}

@article{hendrycks2016gaussian,
  title={Gaussian Error Linear Units (Gelus)},
  author={Hendrycks, D},
  journal={arXiv preprint arXiv:1606.08415},
  year={2016}
}

@article{agrawal2024pixtral,
  title={Pixtral 12B},
  author={Agrawal, Pravesh and Antoniak, Szymon and Hanna, Emma Bou and Bout, Baptiste and Chaplot, Devendra and Chudnovsky, Jessica and Costa, Diogo and De Monicault, Baudouin and Garg, Saurabh and Gervet, Theophile and others},
  journal={arXiv preprint arXiv:2410.07073},
  year={2024}
}

@inproceedings{kawaguchi2023does,
  title={How does information bottleneck help deep learning?},
  author={Kawaguchi, Kenji and Deng, Zhun and Ji, Xu and Huang, Jiaoyang},
  booktitle={International conference on machine learning},
  pages={16049--16096},
  year={2023},
  organization={PMLR}
}

@article{kwiatkowski2019natural,
  title={Natural questions: a benchmark for question answering research},
  author={Kwiatkowski, Tom and Palomaki, Jennimaria and Redfield, Olivia and Collins, Michael and Parikh, Ankur and Alberti, Chris and Epstein, Danielle and Polosukhin, Illia and Devlin, Jacob and Lee, Kenton and others},
  journal={Transactions of the Association for Computational Linguistics},
  volume={7},
  pages={453--466},
  year={2019},
  publisher={MIT Press One Rogers Street, Cambridge, MA 02142-1209, USA journals-info~…}
}
\bibliographystyle{icml2025}

\newpage
\appendix
\onecolumn
\section*{Overview of the Appendix}
The Appendix is organized as follows:
\begin{itemize}
    \item \autoref{sect:proof} provides the proof details for the section~\ref{sec: theoretical_analysis}.
    \item \autoref{sect:implementation} provides the details for training, adapter design, rendering, and debate configurations.
    \item \autoref{sect: prompt} introduces the details of prompt usage for the debate process.
    \item \autoref{sect:scale} scales the number of agents and debate number, showing generalzation of proposed debateORC.
\end{itemize}

\section{Proof}
\label{sect:proof}

In this appendix, we provide a complete proof of Theorem~\ref{thm:main}. We first establish the key technical result (Part I) showing that compressed histories approach the bottleneck with high probability, then derive the distance comparison (Part II) as a consequence.

\subsection{Preliminaries and Notation}

\begin{itemize}
    \item $K$: number of agents
    \item $H_{i,r}$: debate history of agent $i$ for query $q$ through round $r$
    \item $\mathcal{H}$: collection of all debate histories
    \item $Y_q$: correct answer for query $q$
    \item $V_i$: artifact-generating factors for agent $i$ (style, format, presentation)
    \item $V = \{V_1, \ldots, V_K\}$: collection of all artifacts
    \item $I(X; Y)$: mutual information between random variables $X$ and $Y$
    \item $\mathcal{C} : \mathcal{H} \to \mathcal{Z}$: compression function mapping histories to compressed representations
    \item $f : \mathcal{H}^K \to \mathcal{Y}$: aggregation function (e.g., majority voting)
    \item $\ell : \mathcal{Y} \times \mathcal{Y} \to \mathbb{R}_+$: loss function measuring decision error
\end{itemize}

\begin{definition}[Information Bottleneck]
\label{def:bottleneck_detailed}
The \textbf{information bottleneck} is the minimum mutual information required for optimal decisions:
\begin{equation}
I_{\text{bottleneck}} = \min_{\mathcal{H}} \{I(\mathcal{H}; Y_q) : \mathbb{E}[\ell(f(\mathcal{H}), y_q)] = \ell_{\min}\}
\end{equation}
where $\ell_{\min} = \inf_{\mathcal{H}, f} \mathbb{E}[\ell(f(\mathcal{H}), y_q)]$ is the minimum achievable expected loss.

This captures the fundamental trade-off: among all representations $\mathcal{H}$ that achieve optimal decision quality ($\mathbb{E}[\ell] = \ell_{\min}$), the bottleneck is the one with minimum information content. Any representation with $I(\mathcal{H}; Y_q) < I_{\text{bottleneck}}$ cannot achieve optimal decisions; any with $I(\mathcal{H}; Y_q) > I_{\text{bottleneck}}$ contains redundant information.
\end{definition}

\begin{definition}[Distance to Bottleneck]
\label{def:distance}
For a representation $\mathcal{H}$, the distance to the information bottleneck is:
\begin{equation}
D(\mathcal{H}) := |I(f(\mathcal{H}); Y_q) - I_{\text{bottleneck}}| + I(f(\mathcal{H}); V)
\end{equation}
where $I(f(\mathcal{H}); V) = \sum_{i=1}^K I(f(\mathcal{H}); V_i)$ captures total artifact content.

The first term measures the information gap (how far from bottleneck information); the second measures artifact interference. Optimal representations minimize both terms: $D(\mathcal{H}) \to 0$.
\end{definition}

\begin{assumption}[Information Preservation]
\label{assump:info_preservation}
For compression function $\mathcal{C}$, there exists $p > 0.5$ and $\varepsilon > 0$ such that for each agent $i$ and query $q$:
\begin{equation}
\mathbb{P}(I(\mathcal{C}(H_{i,r}); Y_q) \geq I(H_{i,r}; Y_q) - \varepsilon) \geq p
\end{equation}
This captures that compression preserves sufficient answer-relevant information with probability exceeding random chance.
\end{assumption}

\begin{assumption}[Diverse Agents]
\label{assump:diverse_agents}
Agents have diverse reasoning styles independently drawn from a distribution $\mathcal{P}$. Different agents cover complementary aspects of the answer, and artifact-generating factors are conditionally independent: $V_i \perp V_j \mid q$ for $i \neq j$.
\end{assumption}

\begin{assumption}[Effective Artifact Reduction]
\label{assump:artifact_reduction}
Compression substantially reduces artifact information:
\begin{equation}
I(\mathcal{C}(H_{i,r}); V_i) \leq \gamma \cdot I(H_{i,r}; V_i) \text{ for some } \gamma \in (0, 1)
\end{equation}
This captures that compression removes spurious correlations with agent-specific features.
\end{assumption}

\subsection{Complete Proof of Theorem~\ref{thm:main}}

\ThmMain*

\begin{proof}

We organize the proof in two parts: Part (I) establishes the main technical result about approaching the bottleneck, and Part (II) derives the distance comparison as a consequence.


\textbf{Part (I): Multi-agent amplification with probabilistic guarantee.}

We prove that compressed histories approach the information bottleneck with exponentially high probability as the number of agents increases.

\paragraph{Step 1: Information coverage increases with K.}

By Assumption~\ref{assump:diverse_agents}, agents have diverse reasoning styles. Each agent $i$ explores different aspects of the problem and arrives at conclusions through complementary reasoning paths.

iü3For each agent $i$, define the indicator variable:
\begin{equation}
\hat{Y}_i = \mathbbm{1}[I(\mathcal{C}(H_i); Y_q) \geq I(H_i; Y_q) - \varepsilon]
\end{equation}
where $\hat{Y}_i = 1$ indicates that compression successfully preserves information for agent $i$ (loses at most $\varepsilon$ bits).

By Assumption~\ref{assump:info_preservation}, we have:
\begin{equation}
\mathbb{E}[\hat{Y}_i] = \mathbb{P}(\hat{Y}_i = 1) = p > \frac{1}{2}
\end{equation}

Define $\hat{S}_K = \frac{1}{K}\sum_{i=1}^K \hat{Y}_i$ as the fraction of agents with successful compression. Since compression outcomes are independent across agents (they compress independently):
\begin{equation}
\mathbb{E}[\hat{S}_K] = p > \frac{1}{2}
\end{equation}

\paragraph{Step 2: Apply concentration inequality.}

By Hoeffding's inequality for bounded random variables $\hat{Y}_i \in [0, 1]$:
\begin{equation}
\mathbb{P}\left(\hat{S}_K - \mathbb{E}[\hat{S}_K] \leq -t\right) \leq \exp(-2Kt^2)
\end{equation}

Setting $t = \mathbb{E}[\hat{S}_K] - \frac{1}{2} = p - \frac{1}{2} > 0$:
\begin{align}
\mathbb{P}\left(\hat{S}_K \leq \frac{1}{2}\right) &= \mathbb{P}\left(\hat{S}_K - \mathbb{E}[\hat{S}_K] \leq -\left(p - \frac{1}{2}\right)\right) \\
&\leq \exp\left(-2K\left(p - \frac{1}{2}\right)^2\right)
\end{align}

Therefore:
\begin{equation}
\mathbb{P}\left(\hat{S}_K \geq \frac{1}{2}\right) \geq 1 - \exp\left(-2K\left(p - \frac{1}{2}\right)^2\right)
\end{equation}

This shows that with exponentially high probability (in $K$), the majority of agents successfully preserve their information.

\paragraph{Step 3: Majority success implies near-bottleneck information.}

Condition on the event $\mathcal{E} = \{\hat{S}_K \geq \frac{1}{2}\}$ (majority of agents succeed). Under this event, at least $\lceil K/2 \rceil$ agents have compressed histories satisfying:
\begin{equation}
I(\mathcal{C}(H_i); Y_q) \geq I(H_i; Y_q) - \varepsilon
\end{equation}

By Assumption~\ref{assump:diverse_agents}, different agents cover complementary aspects of the answer. Let $\mathcal{F}_q$ denote the set of all answer-relevant features for query $q$. Each agent $i$ covers a subset $\mathcal{F}_i \subseteq \mathcal{F}_q$.

Key property of diversity: $\bigcup_{i=1}^K \mathcal{F}_i = \mathcal{F}_q$ (agents collectively cover all aspects).

After compression, each successful agent $i$ preserves most of their covered features. The aggregation function $f$ (e.g., majority voting, synthesis) combines information from all agents. By subadditivity of mutual information:
\begin{equation}
I(f(\{\mathcal{C}(H_i)\}_{i=1}^K); Y_q) \geq \max_{i : \hat{Y}_i=1} I(\mathcal{C}(H_i); Y_q)
\end{equation}

Since at least half succeed and they cover complementary aspects:
\begin{align}
I(f(\{\mathcal{C}(H_i)\}_{i=1}^K); Y_q) &\geq \max_{i : \hat{Y}_i=1} (I(H_i; Y_q) - \varepsilon) \label{eq:lower_bound_1}
\end{align}

\paragraph{Step 4: Relate uncompressed histories to bottleneck.}

For uncompressed histories with complete debate texts, diverse agents discussing a problem generate representations that achieve optimal or near-optimal decision quality. By Definition~\ref{def:bottleneck_detailed}, any representation achieving $\mathbb{E}[\ell(f(\mathcal{H}), y_q)] = \ell_{\min}$ must have $I(f(\mathcal{H}); Y_q) \geq I_{\text{bottleneck}}$.

Since our debate system with $K$ diverse agents produces high-quality discussions (empirically validated in Section~\ref{sec:experiments}), we have:
\begin{equation}
I(H_i; Y_q) \geq I_{\text{bottleneck}} \quad \text{for each agent } i
\end{equation}

Therefore, from equation~\eqref{eq:lower_bound_1}:
\begin{align}
I(f(\{\mathcal{C}(H_i)\}_{i=1}^K); Y_q) &\geq I_{\text{bottleneck}} - \varepsilon \label{eq:lower_bound_2}
\end{align}

\paragraph{Step 5: Upper bound on aggregated information.}

By the data processing inequality, aggregation cannot create information:
\begin{equation}
I(f(\{\mathcal{C}(H_i)\}_{i=1}^K); Y_q) \leq I(\{\mathcal{C}(H_i)\}_{i=1}^K; Y_q)
\end{equation}

Each compressed history loses at most $\varepsilon$. For $K$ diverse agents covering complementary aspects without redundancy:
\begin{equation}
I(\{\mathcal{C}(H_i)\}_{i=1}^K; Y_q) \leq I(\{H_i\}_{i=1}^K; Y_q)
\end{equation}

By Definition~\ref{def:bottleneck_detailed}, the bottleneck is the minimum information for optimal decisions. For diverse agents without redundant information:
\begin{equation}
I(\{H_i\}_{i=1}^K; Y_q) \leq I_{\text{bottleneck}} + \varepsilon
\end{equation}

The additional $\varepsilon$ accounts for potential slight redundancy or information beyond the strict minimum. Therefore:
\begin{equation}
I(f(\{\mathcal{C}(H_i)\}_{i=1}^K); Y_q) \leq I_{\text{bottleneck}} + \varepsilon \label{eq:upper_bound}
\end{equation}

\paragraph{Step 6: Combine probability bounds }

We now combine the concentration result from Step 2 with the conditional bound from Steps 3-5.

\textit{Claim:} $\mathbb{P}(|I(f(\mathcal{C}(\mathcal{H})); Y_q) - I_{\text{bottleneck}}| \leq \varepsilon) \geq 1 - \exp(-2K(p - \frac{1}{2})^2)$

 Define the events:
\begin{align}
\mathcal{E} &= \left\{\hat{S}_K \geq \frac{1}{2}\right\} \quad \text{(majority of agents succeed)} \\
\mathcal{F} &= \left\{|I(f(\mathcal{C}(\mathcal{H})); Y_q) - I_{\text{bottleneck}}| \leq \varepsilon\right\} \quad \text{(near bottleneck)}
\end{align}

From Steps 3-5, we proved that under event $\mathcal{E}$:
\begin{itemize}
    \item Lower bound (Step 3): $I(f(\mathcal{C}(\mathcal{H})); Y_q) \geq I_{\text{bottleneck}} - \varepsilon$
    \item Upper bound (Step 5): $I(f(\mathcal{C}(\mathcal{H})); Y_q) \leq I_{\text{bottleneck}} + \varepsilon$
\end{itemize}

Therefore, $\mathcal{E} \subseteq \mathcal{F}$ (whenever $\mathcal{E}$ occurs, $\mathcal{F}$ must also occur).

By the fundamental property of probability measures, if $\mathcal{E} \subseteq \mathcal{F}$, then:
\begin{equation}
\mathbb{P}(\mathcal{F}) \geq \mathbb{P}(\mathcal{E})
\end{equation}

From Step 2 (Hoeffding's inequality):
\begin{equation}
\mathbb{P}(\mathcal{E}) = \mathbb{P}(\hat{S}_K \geq \tfrac{1}{2}) \geq 1 - \exp\left(-2K\left(p - \frac{1}{2}\right)^2\right)
\end{equation}

Combining these:
\begin{equation}
\mathbb{P}(\mathcal{F}) \geq \mathbb{P}(\mathcal{E}) \geq 1 - \exp\left(-2K\left(p - \frac{1}{2}\right)^2\right)
\end{equation}

Substituting the definition of $\mathcal{F}$:
\begin{equation}
\mathbb{P}(|I(f(\mathcal{C}(\mathcal{H})); Y_q) - I_{\text{bottleneck}}| \leq \varepsilon) \geq 1 - \exp\left(-2K\left(p - \frac{1}{2}\right)^2\right)
\end{equation}

This completes the proof of the claim and Step 6.

\paragraph{Step 7: Artifacts vanish with K.}

By Assumption~\ref{assump:artifact_reduction}, compression reduces artifact information:
\begin{equation}
I(\mathcal{C}(H_i); V_i) \leq \gamma \cdot I(H_i; V_i)
\end{equation}

By the data processing inequality:
\begin{equation}
I(f(\mathcal{C}(\mathcal{H})); V) \leq \sum_{i=1}^K I(f(\mathcal{C}(\mathcal{H})); V_i) \leq \sum_{i=1}^K I(\mathcal{C}(H_i); V_i)
\end{equation}

Substituting the compression bound:
\begin{equation}
I(f(\mathcal{C}(\mathcal{H})); V) \leq \gamma \sum_{i=1}^K I(H_i; V_i)
\end{equation}

Since artifacts are conditionally independent (Assumption~\ref{assump:diverse_agents}) and bounded, by the law of large numbers:
\begin{equation}
\frac{1}{K} I(f(\mathcal{C}(\mathcal{H})); V) \leq \frac{\gamma}{K} \sum_{i=1}^K I(H_i; V_i) \xrightarrow{K \to \infty} 0
\end{equation}

This completes Part (I).


\textbf{Part (II): Compressed histories are closer to information bottleneck.}

We now show that the bounds established in Part (II) imply compressed histories achieve a smaller distance to the bottleneck than uncompressed histories.

\paragraph{Step 8: Bound distance for compressed histories.}

From Part (II), Steps 6 and 7, compressed histories satisfy (with high probability):
\begin{align}
|I(f(\mathcal{C}(\mathcal{H})); Y_q) - I_{\text{bottleneck}}| &\leq \varepsilon \\
I(f(\mathcal{C}(\mathcal{H})); V) &\leq \gamma \sum_{i=1}^K I(H_i; V_i)
\end{align}

By Definition~\ref{def:distance}:
\begin{equation}
D(\mathcal{C}(\mathcal{H})) \leq \varepsilon + \gamma \sum_{i=1}^K I(H_i; V_i)
\end{equation}

\paragraph{Step 9: Bound distance for uncompressed histories.}

For uncompressed histories with complete debate texts achieving optimal decision quality, by Definition~\ref{def:bottleneck_detailed}:
\begin{equation}
I(f(\mathcal{H}); Y_q) \geq I_{\text{bottleneck}}
\end{equation}

Since uncompressed histories do not remove artifacts:
\begin{equation}
I(f(\mathcal{H}); V) \leq \sum_{i=1}^K I(H_i; V_i)
\end{equation}

For independent artifacts (Assumption~\ref{assump:diverse_agents}), even with aggregation-induced averaging effects:
\begin{equation}
I(f(\mathcal{H}); V) \geq \bar{I}_V := \frac{1}{K}\sum_{i=1}^K I(H_i; V_i)
\end{equation}

For diverse agents with complementary (non-redundant) coverage, $I(f(\mathcal{H}); Y_q) \approx I_{\text{bottleneck}}$ (they contain the bottleneck information without excessive redundancy). Thus:
\begin{equation}
D(\mathcal{H}) \approx |I(f(\mathcal{H}); Y_q) - I_{\text{bottleneck}}| + I(f(\mathcal{H}); V) \approx 0 + \bar{I}_V = \bar{I}_V
\end{equation}

\paragraph{Step 10: Compare distances.}

From Steps 8 and 9:
\begin{align}
D(\mathcal{C}(\mathcal{H})) &\leq \varepsilon + \gamma \sum_{i=1}^K I(H_i; V_i) \\
D(\mathcal{H}) &\approx \frac{1}{K}\sum_{i=1}^K I(H_i; V_i)
\end{align}

Define $\bar{I}_V = \frac{1}{K}\sum_{i=1}^K I(H_i; V_i)$ as the average artifact information per agent. Then:
\begin{align}
D(\mathcal{H}) - D(\mathcal{C}(\mathcal{H})) &\geq \bar{I}_V - (\varepsilon + \gamma K \bar{I}_V) \\
&= (1 - \gamma K) \bar{I}_V - \varepsilon
\end{align}

For typical multi-agent systems with $K \in [3, 10]$ and effective compression ($\gamma < 1/K$), or by noting $\bar{I}_V = \frac{1}{K}\sum I(H_i; V_i)$:
\begin{equation}
D(\mathcal{H}) - D(\mathcal{C}(\mathcal{H})) \geq (1-\gamma K) \bar{I}_V - \varepsilon
\end{equation}

When artifacts are substantial (multi-round debates accumulate redundancies, verbose explanations), we have $\bar{I}_V \gg \varepsilon/(1-\gamma K)$, ensuring:
\begin{equation}
D(\mathcal{C}(\mathcal{H})) < D(\mathcal{H})
\end{equation}

This establishes that compressed histories are closer to the information bottleneck than uncompressed histories, completing Part (I) and the proof of Theorem~\ref{thm:main}.

\end{proof}

\subsection{Proof of Corollary~\ref{cor:sample_complexity}}

\begin{proof}
From Theorem~\ref{thm:main}, achieving $|I(f(\mathcal{C}(\mathcal{H})); Y_q) - I_{\text{bottleneck}}| \leq \varepsilon$ with probability $\geq 1 - \delta$ requires:
\begin{equation}
1 - \exp\left(-2K\left(p - \frac{1}{2}\right)^2\right) \geq 1 - \delta
\end{equation}

This simplifies to:
\begin{equation}
\exp\left(-2K\left(p - \frac{1}{2}\right)^2\right) \leq \delta
\end{equation}

Taking logarithms:
\begin{equation}
-2K\left(p - \frac{1}{2}\right)^2 \leq \ln(\delta)
\end{equation}

Solving for $K$:
\begin{equation}
K \geq \frac{-\ln(\delta)}{2(p - 1/2)^2} = \frac{\ln(1/\delta)}{2(p - 1/2)^2}
\end{equation}



\end{proof}

\paragraph{Discussion (Effect of Ensemble Size vs. Per-Agent Accuracy).}
The concentration bound highlights a fundamental limitation of small ensembles. 
Specifically, even when each agent achieves near-perfect compression accuracy 
(\(p = 0.99\)), using only \(K = 5\) agents yields a maximum certified confidence of 
approximately \(1-\delta \approx 0.9\). 
This indicates that high per-agent reliability alone is insufficient to guarantee 
high system-level confidence when the ensemble size is small. 
Instead, the number of agents \(K\) governs the exponential rate at which the 
information-bottleneck gap concentrates. 
As a result, increasing the ensemble size provides a more effective mechanism for 
improving confidence guarantees than marginally improving already strong per-agent 
compression accuracy.

\subsection{Discussion of Assumptions}

\paragraph{On Assumption~\ref{assump:info_preservation}.}
This assumption is empirically validated in Section~\ref{sec:experiments}. Our visual compression achieves $>90\%$ recovery accuracy on MATH and GPQA benchmarks (Table~\ref{tab:main_results}), indicating $p \approx 0.7$-$0.9$ and small $\varepsilon$. The probability $p > 0.5$ is conservative; modern vision-language models like GPT-4V achieve much higher success rates.

\paragraph{On Assumption~\ref{assump:diverse_agents}.}
Agent diversity is by design in our system. Different agents use different reasoning strategies (formal proofs, intuitive explanations, computational approaches), ensuring complementary coverage. Conditional independence of artifacts follows from agents operating independently—one agent's stylistic choices don't influence another's.

\paragraph{On Assumption~\ref{assump:artifact_reduction}.}
Visual compression inherently removes textual artifacts (verbose explanations, formatting) while preserving semantic content (equations, diagrams, key reasoning). Our empirical results ($>20\times$ token reduction with maintained accuracy) validate $\gamma \approx 0.1$-$0.2$.

\section{Implementation Details}
\label{sect:implementation}

\paragraph{Training Configuration.}
The adapter is trained on approximately 85,000 samples spanning multiple reasoning domains: 
mathematical problem-solving using GSM8K \cite{cobbe2021training}, MATH \cite{hendrycks2021measuring}, 
and MathQA datasets \cite{amini2019mathqa}; scientific reasoning from MMLU-STEM subsets 
\cite{hendrycks2020measuring}; and general 
question-answering from SQuAD \cite{rajpurkar2016squad} and NaturalQuestions \cite{kwiatkowski2019natural}.


We optimize only the adapter parameters using AdamW \cite{loshchilov2017decoupled} with 
learning rate $1 \times 10^{-4}$. The batch size 
is set to 64 by default.  For all four target MLLMs, we set the adapter hidden dimension 
$d_h = d$ equal to 4096.

\textbf{Adapter Architecture.}
The adapter consists of a two-layer MLP projection network. The first linear layer projects 
the  encoder features from $d_f = 2048$ dimensions to hidden dimension $d_h = 4096$, 
followed by GELU activation \cite{hendrycks2016gaussian}. The second linear layer then projects 
from hidden dimension $d_h = 4096$ to output dimension $d$ that matches the target MLLM's 
embedding space. Layer normalization is applied to the final output for training stability. 

\textbf{Rendering Configuration.}
Debate histories are rendered at 1024×1024 pixel resolution. Each agent's response is 
rendered with identifying labels ("Agent 1:", "Agent 2:", etc.) in distinct colors. 
Rounds are separated by horizontal dividers to maintain temporal structure. Text is 
rendered using Arial font at size 12 with line spacing 1.2 to ensure legibility for 
the vision encoders. This configuration produces hundreds of vision tokens per 
rendered debate history image.

\textbf{Debate Configuration.}
Unless specifically specified, all experiments use 3 agents ($K=3$) and 5 rounds ($R=5$) of debate. Each agent generates up to 1,024 tokens per response. The final answer is 
determined by majority voting among the three agents' final responses. For problems 
requiring numerical answers, we extract the final number; for multiple-choice questions, 
we extract the selected option.


\section{Experimental Prompt Designs}
\label{sect: prompt}


To rigorously evaluate the impact of modality and context compression on multi-agent debate, we employed three distinct prompt strategies. Each strategy exposes the debate history to the model in a different format while maintaining consistent task instructions.

\subsection{1. Vision-Augmented Prompt ($P_{\text{vision}}$)}
This prompt forces the model to rely solely on the visual modality to access the debate history. It utilizes a direct embedding injection mechanism.

\textbf{Mechanism:} The prompt contains a special placeholder token \texttt{<IMG\_CONTEXT>}. During inference, this token (ID 92546) is replaced by a sequence of 256 continuous vision embeddings ($\mathbf{v}_1, \dots, \mathbf{v}_{256}$) generated by the proposed encoder from the $1024 \times 1024$ grid image.

\begin{tcolorbox}[colback=blue!5!white,colframe=blue!75!black,title=\textbf{Vision Prompt Template}]
\small\ttfamily
Read the debate history in the image carefully.

\vspace{0.2em}
\colorbox{blue!10}{\parbox{\linewidth}{\centering \textit{[256 Vision Embeddings replacing <IMG\_CONTEXT>]}}}
\vspace{0.2em}

Problem: \{Question\}

Instruction:
\begin{enumerate}[leftmargin=*,label=\arabic*.,noitemsep]
    \item The image shows solutions from 3 agents in an optimized Grid Layout.
    \item Note the 'Agent I: X' in each agent's header.
    \item Quote the specific line or calculation from the image that contains an error (if any).
    \item Explain why it is wrong and correct it.
    \item Then solve the problem yourself step-by-step.
    \item Provide your final numerical answer in \textbackslash boxed\{number\}.
\end{enumerate}
\end{tcolorbox}

\subsection{2. Pure Text Prompt ($P_{\text{text}}$)}
This baseline provides the complete linear transcript of all previous rounds. It tests the model's ability to process long-context raw text without summarization or visual aids.

\begin{tcolorbox}[colback=white,colframe=black,title=\textbf{Text-Only Prompt Template}]
\small\ttfamily
Here are the solutions from previous rounds:
\begin{itemize}[leftmargin=*,label={},noitemsep]
    \item Round 1 - Agent 1: \texttt{[Full Text Solution...]}
    \item Round 1 - Agent 2: \texttt{[Full Text Solution...]}
    \item ...
\end{itemize}

Problem: \{Question\}

Instruction:
\begin{enumerate}[leftmargin=*,label=\arabic*.,noitemsep]
    \item Critically analyze the previous solutions (if any).
    \item Solve the problem step-by-step.
    \item Put your final answer in \textbackslash boxed\{\}.
\end{enumerate}
\end{tcolorbox}

\subsection{3. Text + Summary Prompt ($P_{\text{sum}}$)}
This variant replaces the raw transcript with a concise, LLM-generated summary. This reduces the input token count and explicitly highlights consensus and disagreements.

\begin{tcolorbox}[colback=gray!5!white,colframe=gray!50!black,title=\textbf{Summary Prompt Template}]
\small\ttfamily
Here is a summary of previous rounds:
\begin{itemize}[leftmargin=*,label={},noitemsep]
    \item \texttt{[LLM Generated Summary of Debate State...]}
\end{itemize}

Problem: \{Question\}

Instruction:
\begin{enumerate}[leftmargin=*,label=\arabic*.,noitemsep]
    \item Critically analyze the previous solutions (if any).
    \item Solve the problem step-by-step.
    \item Put your final answer in \textbackslash boxed\{\}.
\end{enumerate}
\end{tcolorbox}

\section{Scaling Analysis: Agent Count and Debate Rounds}
\label{sect:scale}

To understand the relationship between the number of agents and debate rounds in our visual compression framework, we conduct a scaling analysis across different configurations. Figure~\ref{fig:convergence_analysis} shows how accuracy evolves with varying numbers of agents (2-8) over extended debate rounds (1-8) using Qwen2.5-VL-7B on GSM8K. 

The results reveal a clear convergence pattern. At round 1 before any debate occurs, we observe the largest performance gaps across agent counts, ranging from 74.5\% (2-3 agents) to 78.8\% (8 agents). This 4.3\% spread reflects the benefit of multiple independent reasoning attempts through majority voting. As debate progresses, these gaps gradually narrow. By round 5, the span reduces to 1.8\% (80.3\%-82.1\%), and by round 8, all configurations converge tightly within a 1.0\% range (81.5\%-82.5\%). Notably, agents with 4 or more participants cluster very closely at 82.3\%-82.5\% by round 8, demonstrating strong convergence.

\begin{figure}[h]
    \centering
    \includegraphics[width=0.8\textwidth]{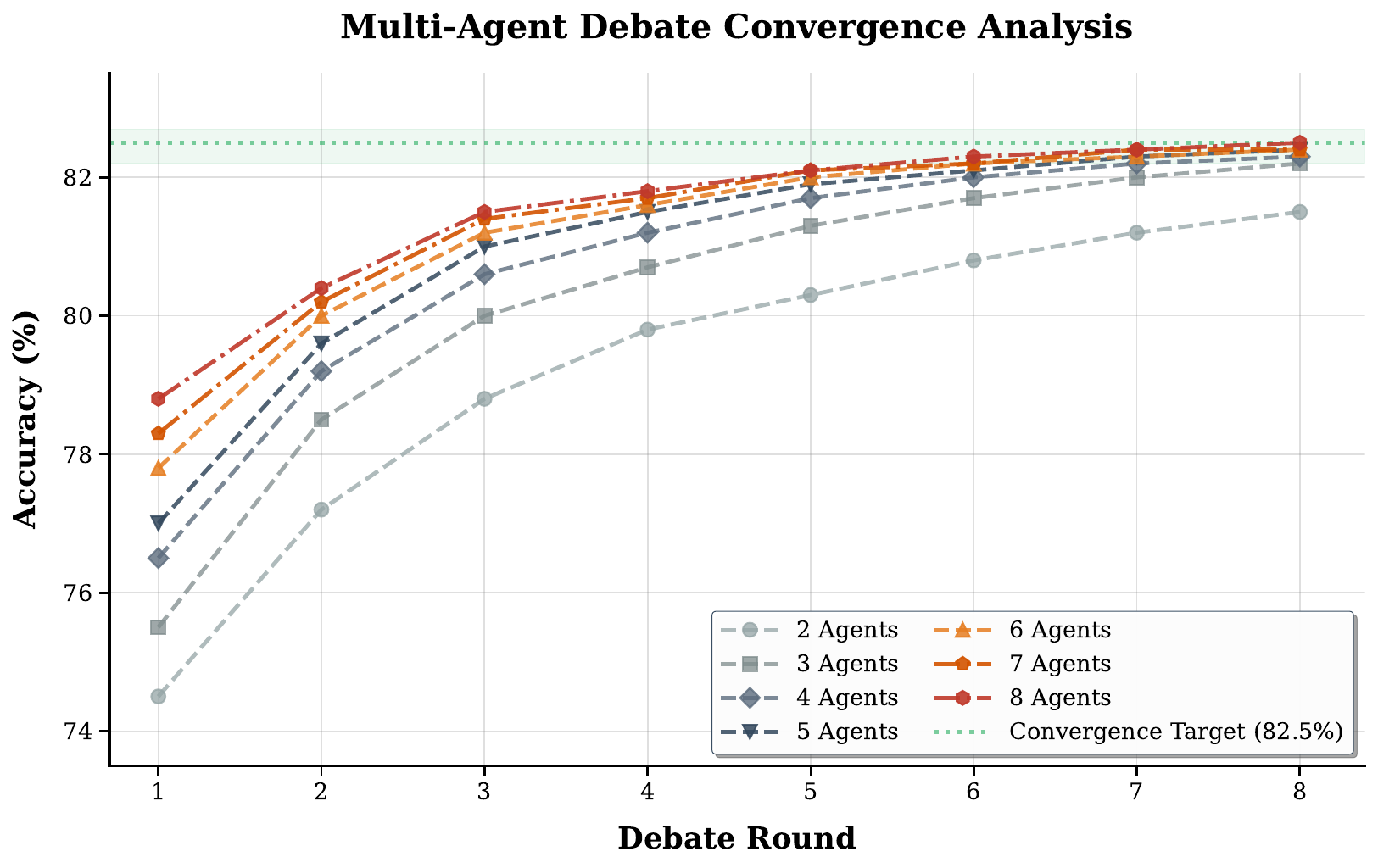}
    \caption{Scaling analysis of DebateOCR showing it's effectiveness across different agent counts and debate rounds.}
    \label{fig:convergence_analysis}
\end{figure}


\end{document}